\documentclass[acmsmall]{acmart}
\usepackage{wrapfig}
\usepackage{graphicx}
\usepackage{epstopdf}
\usepackage[ruled,linesnumbered]{algorithm2e}
\usepackage{algorithmic}
\usepackage{paralist} 
\usepackage{color}
\usepackage{multirow}
\usepackage{rotating}
\usepackage[mathscr]{eucal} 
\usepackage{amsbsy} 
\usepackage{subfigure}
\usepackage{booktabs}
\usepackage{xcolor}
\usepackage{tikz}
\usepackage{multirow}

\usepackage{balance}
\usepackage{tablefootnote}
\usepackage{empheq}
\usepackage{moresize}
\usepackage[inline]{enumitem}
\setlist{nolistsep}
\usepackage{sidecap}
\usepackage{mdframed}
\newcommand{\tensor}[1]{\underline{ \mathbf{#1} }}
\usepackage{caption}

\newcounter{ALC@tempcntr}
\newcommand{\ourproblem}{\textit{Trapped Under Ice}\xspace}
\newcommand{\method}{\textsc{IceBreaker}\xspace}
\newcommand{\ibPlus}{\textsc{IceBreaker\small{$++$}}\xspace}
\newcommand{\codeurl}{\url{https://github.com/ravdeep003/adaptive-granularity-tensors}\space}

\AtBeginDocument{%
  \providecommand\BibTeX{{%
    \normalfont B\kern-0.5em{\scshape i\kern-0.25em b}\kern-0.8em\TeX}}}

\setcopyright{acmcopyright}
\copyrightyear{2021}
\acmYear{2021}
\acmDOI{10.1145/1122445.1122456}

\begin{document}
	\title{Adaptive Granularity in Tensors }
    \subtitle{A Quest for Interpretable Structure}

\author{Ravdeep S Pasricha} \author{Ekta Gujral} \author{Evangelos E. Papalexakis}
\affiliation{%
  \institution{Department of Computer Science and Engineering, University of California Riverside}
  \streetaddress{900 University Avenue}
  \city{Riverside}
  \state{California}
  \country{USA}
  \postcode{92521}
  }
\email{rpasr001@ucr.edu}
\email{egujr001@ucr.edu}
\email{epapalex@cs.ucr.edu}

\renewcommand{\shortauthors}{R. Pasricha, et al.}

	\begin{abstract}
Data collected at very frequent intervals is usually extremely sparse and has no structure that is exploitable by modern tensor decomposition algorithms. Thus the utility of such tensors is low, in terms of the amount of interpretable and exploitable structure that one can extract from them. In this paper, we introduce the problem of finding a tensor of adaptive aggregated granularity that can be decomposed to reveal meaningful latent concepts (structures) from datasets that, in their original form, are not amenable to tensor analysis. Such datasets fall under the broad category of sparse point processes that evolve over space and/or time. To the best of our knowledge, this is the first work that explores adaptive granularity aggregation in tensors. Furthermore, we formally define the problem and discuss what different definitions of ``good structure'' can be in practice, and show that optimal solution is of prohibitive combinatorial complexity. Subsequently, we propose an efficient and effective greedy algorithm called \method, which follows a number of intuitive decision criteria that locally maximize the ``goodness of structure'', resulting in high-quality tensors. We evaluate our method on synthetic, semi-synthetic  and real datasets. In all the cases, our proposed method constructs tensors that have very high structure quality. 
\end{abstract}

	\maketitle

	\section{Introduction}
\label{sec:intro}
In the age of big data, applications deal with data collected at very fine-grained time intervals. In many real world applications, the data collected spans long periods of time and can be extremely sparse. For instance, a time-evolving social network that records interactions of users every second results in a very sparse adjacency matrix if observed at that granularity. Similarly, in spatio-temporal data, if one considers GPS data over time, discretizing GPS coordinates based on the observed granularity can lead to very sparse data which may not contain any visible and useful structure.
 How can we find meaningful and actionable structure in these types of data?  A great deal of such datasets are multi-aspect in nature and hence can be modeled using tensors. For instance, a three-mode tensor can represent a time-evolving graph capturing user-user interactions over a period of time, measuring crime incidents in a city community area over a period of time \cite{frosttdataset2017}, or measuring traffic patterns \cite{zheng2014urban}.  Tensor decomposition has been used in order to extract hidden patterns from such multi-aspect data \cite{sidiropoulos2017tensor, papalexakis2016tensors,kolda2009tensor}. However the degree of sparsity in the tensor, which is a function of the granularity in which the tensor is formed, significantly affects the ability of the decomposition to discover ``meaningful'' structure in the data.
 
\begin{figure}[t]
	\begin{center}
		\includegraphics[width = 0.90\textwidth]{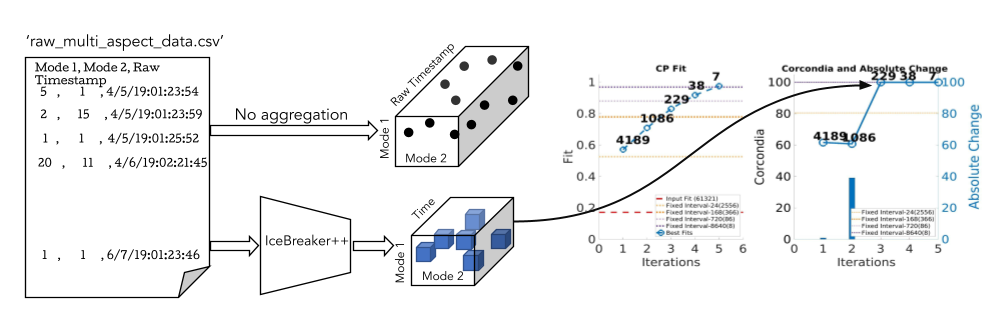}
		\caption{Starting from raw CSV files, \method discovers a tensor that has good structure (under various measures of quality, including interpretability and predictive quality), outperforming traditional fixed aggregation heuristics. Furthermore, \method using various notions of locally optimal structure, discovers different resolutions in the data.}%
		
		\label{fig:crown}
	\end{center}

\end{figure}
Consider a dataset which can be modeled as three-mode tensor, where the third mode is temporal as shown in Figure \ref{fig:crown}. If the granularity of the temporal mode is too fine (in milliseconds or seconds), one might end up with a tensor that is extremely long on the time mode and where each instance of time has very small number of entries. This results in a extremely sparse tensor, which typically is of very high rank, and which usually has no underlying exploitable structure for widely popular and successful tensor decomposition algorithms \cite{sidiropoulos2017tensor,papalexakis2016tensors,kolda2009tensor}. However, as we aggregate data points over time, exploitable structure starts to appear (where-by ``exploitable'' we define the kind of low-rank structure that a tensor decomposition can successfully model and extract). In this paper we set out to explore what is the best such data-driven aggregation of a tensor which leads to better, exploitable, and interpretable structure, and how this fares against the traditional alternative of selecting a fixed interval for aggregation.

As far as tackling the problem above, there is considerable amount of work that focuses on a special case, that of aggregating edges of a time evolving graph into ``mature'' adjacency matrices based on certain graph properties \cite{soundarajan2016generating,sulo2010meaningful,sun2007graphscope}.
In our work, however, we address the problem in more general terms, where the underlying data can be any point process that is observed over time and/or space, and where the aggregation/discretization of the corresponding dimensions directly affects our ability to extract interpretable patterns via tensor decomposition. Effectively, as shown in Figure \ref{fig:crown}, in this paper we work towards automating the data aggregation starting from raw data into a well-structured tensor. This paper is based on the preliminary work which has appeared in arxiv \cite{pasricha2019adaptive} and non archival workshop \cite{mlg2020_35}.

Our contributions in this work are as follows:
\begin{itemize}
	\item \textbf{Novel Problem Formulation}: We formally define the problem of optimally aggregating a tensor, which is formed from raw sparse data in their original level of aggregation, into a tensor with exploitable and interpretable structure. We further show that solving this problem optimally is computationally intractable. To the best of our knowledge, this paper is the first to tackle this problem in its general form, and we view our formulation as the first step towards automating the process of creating well-behaved tensor datasets.
	\item \textbf{Practical Algorithm}: We propose a practical, efficient, and effective algorithm that is able to produce high-quality tensors from raw data without incurring the combinatorial cost of the optimal solution. Our proposed method follows a greedy approach, where at each step we decide whether different ``slices'' of the tensor are aggregated based on a variety of intuitive functions that characterize the ``goodness of structure'' locally.
	
	\item \textbf{Experimental Evaluation}: We extensively evaluate our proposed method on synthetic, semi-synthetic and real data where we use popular heuristic measures of structure goodness to measure success. Furthermore, we conduct a data mining case study on a large real dataset of crime over time in Chicago, where we identify interpretable hidden patterns in multiple time resolutions.
\end{itemize}

We make our implementation publicly available\footnote{\codeurl} in order to encourage reproducibility of our results.

	\section{Problem Formulation}
\label{sec:problem}

\subsection{Tensor Definition and Notations}
Tensors are multi-dimensional extensions of matrices, and tensor decompositions are a class of methods that extract latent structure from tensor datasets by extending techniques such as Principal Component Analysis and Singular Value Decomposition. The different ``dimensions'' of a tensor are usually referred to as ``modes''. In this paper, we focus on the CANDECOMP/PARAFAC (henceforth refered to as CP for brevity) decomposition \cite{carroll1970analysis, harshman1970foundations}, which is the ``rank decomposition'' of a tensor, i.e., the decomposition of an arbitrary tensor into a sum of $R$ rank-one tensors. Mathematically, for a three-mode tensor $\tensor{X}$, the CP decomposition is $\tensor{X} \approx \displaystyle{ \sum_{r=1}^R  \mathbf{A}(:,r)\circ \mathbf{B}(:,r)\circ\mathbf{C}(:,r)  }$, where $\circ$ is the generalized outer product. Matrices $\mathbf{A,B,C}$ are called ``factor matrices'', and each column corresponds to a latent pattern, directly relating an entity of the corresponding mode to a value that can be roughly construed as a soft clustering coefficient \cite{papalexakis2012k}. CP has arguably been the most popular tensor decomposition model in applications where the interest is to extract interpretable patterns for exploratory analysis, and thus, we adopt this decomposition model as our standard in this work. In the interest of space, we refer the reader to a number of available surveys \cite{sidiropoulos2017tensor,papalexakis2016tensors,kolda2009tensor}.
We denote tensors as $\tensor{X}$ and matrices as $\mathbf{X}$, and we adopt Matlab-like notation for indexing.

\subsection{Tensor decomposition quality}
Unsupervised tensor decomposition, albeit very popular, poses a significant challenge: how can we tell whether a computed decomposition is of ``high quality'', and how can we go about defining ``quality'' in a meaningful way? Unfortunately, this happens to be a very hard problem to solve \cite{papalexakis2016automatic}, 
and defining a new measure of quality is beyond the scope of this paper. However, there has been significant amount of work in that direction, which basically boils down to 1) model-based measures, where the quality is measured by how well a given decomposition represents the intrinsic hidden structure of the data, and 2) extrinsic measures, where the quality is measured by how well the computed decomposition factors perform in a predictive task. However, extrinsic measures do not generalize, as they specialize to a particular labeled task, and in general we cannot assume that labels will be available for the data at hand. Thus, in this work we focus  on model-based measures, which can provide a general solution.

In model-based measures, the most straightforward one is the fit, i.e., how well does the decomposition approximate the data under the chosen loss function, in a {\em low rank}. Low rank is key, because the number of components (rank) has to be as small and compact as possible in order to lend itself to human evaluation and exploratory analysis. However, fit has been shown to be unstable and prone to errors especially in real and noisy data, thus the community has collectively turned its attention to more robust measures such as the Core Consistency Diagnostic (CORCONDIA for short) \cite{bro2003new}, which measures how well the computed factors obey the CP model.

Both types of quality measure capture different elements of what an end-user would deem good in a set of decomposition factors. In this paper, we are going to use such popular measures of quality in order to characterize the quality of a given tensor dataset $\tensor{X}$. In order to do so, we assume that we have a function $\mathcal{Q}\left( \tensor{X} \right)$ which, optimizes the quality measure $q \left( \right)$ for a given tensor over all possible decomposition ranks $R$ \footnote{In practice, this is done over a small number of low ranks, since low-rank structure is desirable.}, i.e.,

\[
	\mathcal{Q}_i\left( \tensor{X} \right) = \max_{R} q_i \left( \tensor{X}, \mathbf{A,B,C} \right)
\]
where $\mathbf{A,B,C}$ are the $R$-column factor matrices for $\tensor{X}$.
Finally, a useful operation is the $n$-mode product, where a matrix $\mathbf{W}$ is multiplied by the $n$-th mode of a tensor (predicated on matching dimensions in the $n$-th mode of the tensor and the rows of the matrix), denoted as $\tensor{X}\times_n \mathbf{W}$. For instance, an $I\times J\times K$ tensor where $n=3$ and $\mathbf{W}$ of size ${K\times K^*}$, the product $\tensor{X}\times_n \mathbf{W}$ multiplies all third mode slices of $\tensor{X}$ with $\mathbf{W}$ and results in a $I\times J \times K^*$ tensor.
\subsection{The \ourproblem problem}

To give reader an intuition of the problem, consider an example of time-evolving graph which captures social activity  over the span of some time. This example can be modeled as three-mode tensor $\tensor{X}$ of dimensions $I \times J \times K$  where ``sender'' and ``receiver'' are the first two modes, ``time'' being the third mode, and non-zero entry in the tensor represents communications between user at a particular time. If the time granularity is extremely fine-grained (milliseconds or seconds), there might be only handful of data points at a particular time stamp causing resulting tensor to be extremely sparse and to have a high tensor-rank as a result. In that case, $\tensor{X}$ might not have any interpretable low-rank structure that can be exploited by CP. In this example we assume that the third mode (time mode) is too fine-grained but in reality any mode (one or more) can be extremely fine grained. For example, in spatio-temporal data, where the first two modes are latitude and longitude and the third mode is time, all three modes can suffer from the same problem. 

Given tensor $\tensor{X}$ which is created using the ``raw'' granularities,  how does one find a tensor (say $\tensor{Y}$) which has better exploitable structure and hence can be decomposed into meaningful latent structure. This, is informally the \ourproblem problem that we define here (which draws an analogy between the good structure that may exist within the data as being trapped under the ice and not visible by mere inspection). \ourproblem has an inherent assumption that the mode in which we aggregate is ordered (e.g., representing time or space), thus permuting the third mode will lead to a different instance of the problem.

More formally we define our problem as follows:
\begin{mdframed}[linecolor=red!60!black,backgroundcolor=gray!20,linewidth=2pt,topline=false,rightline=false, leftline=false]
Given a tensor $\tensor{X}$ of dimensions $I \times J \times K$ 
Find:\\
A tensor $\tensor{Y}$ of dimensions $I \times J \times K^*$ with $K^* \le K$  such that

\[
	\max_{\mathbf{W}}  \mathcal{Q} \left(   \tensor{X} \times_{3} \mathbf{W}  \right)
\]	
where $\mathcal{Q}$ is a measure of goodness and $\mathbf{W}(i,j) = 1$ if slice $i$ in tensor $\tensor{X}$ is aggregated into slice $j$ in the resulting tensor, otherwise  $\mathbf{W}(i,j) = 0$.

\end{mdframed}
At first glance, \ourproblem might look like a problem amenable to dynamic programming, since it exhibits the optimal substructure property. However, it lacks the overlapping subproblems property:  there are overlapping subproblems across the set of different $\mathbf{W}$ matrices (e.g., two different matrices may have overlapping subproblems) but not within any single $\mathbf{W}$. Thus, we still have to iterate over $2^{K-1}$ $\mathbf{W}$'s refer subsection \ref{hardness} for more details.

\textbf{Structure of $\mathbf{W}$}: The matrix $\mathbf{W}$ has a special structure. Here we provide an example. Consider a three-mode tensor $\tensor{X}$ of dimensions $10\times10\times10$, with the third mode being the time mode. Suppose that the optimal level of aggregation for $\tensor{Y}$ is $K^* = 3$.
 
In this case, $\mathbf{W}$ is of size $3\times 10$ and an example of such matrix is:
\[\mathbf{W} = \begin{bmatrix}
1&1&1&0&0&0&0&0&0&0\\
0&0&0&1&1&1&0&0&0&0\\
0&0&0&0&0&0&1&1&1&1\\
\end{bmatrix}\]
This $\mathbf{W}$ aggregates first three slice of $\tensor{X}$ to form first slice of $\tensor{Y}$, next three to form the second slice and last four to form the third slice. 
No two $\mathbf{W}$ matrices will produce the same aggregation. They can have the same $K^*$ but order of aggregation of slices will be different.
\subsection{Solving \ourproblem optimally is hard}\label{hardness}
Solving \ourproblem optimally poses a number of hurdles. First and foremost, the hardness of the problem depends on the definition of function $\mathcal{Q}$, and most reasonable and intuitive such definitions are very hard to optimize since they are non-differentiable, non-continuous, and not concave. So far, in the literature, to the best of our knowledge, there are only heuristics for this quality function. Even so, those heuristic functions can only be evaluated on a single already fully-aggregated tensor, not a partially aggregated version thereof. Thus, \ourproblem can be only solved optimally via enumerating all admissible solutions and choosing the best. In order to conduct this enumeration, we need to calculate the cardinality of the set of all $\mathbf{W}$ for a given instance of the problem.

\begin{lemma}
	\label{lemma1}
	For an instance of a problem with $K$ initial slices, the cardinality of the set of all $\mathbf{W}$ is $2^{K-1}$
\end{lemma}

\begin{proof}
	To get $K^*$ aggregated slices there are $\binom{K-1}{K^*-1}$ ways to choose each of them leading to a different $\mathbf{W}$. This is a number of ways that $K-1$ partition slots can be filled partitioned by ${K^*-1}$ blocks. In order to get the final number, we need to sum up over all potential $K^*$:
	\[
	\sum_{K^*=0}^{K-1} \binom{K-1}{K^*} = 2^{K-1}
	\]

\end{proof}

Direct corollary of the above lemma is that solving optimally \ourproblem requires calling the function $\mathcal{Q}$ $O\left( 2^K \right)$ times, which is computationally intractable. There may be small room for improvement by exploiting special structure in the set of all $\mathbf{W}$, however, given discontinuities in our objective function $\mathcal{Q}$, this is not be a feasible alternative either. In this paper we define proxy quality functions $\mathcal{Q}$ that lend themselves to partial evaluation on a partially aggregated solution, thus allowing for efficient algorithms
 Thus, in the next section we propose a fast greedy approach which locally optimizes different criteria quality.

	\section{Proposed Methods}
\label{sec:method}

In this section, we propose our efficient and effective greedy algorithm called \method which takes a tensor $\tensor{X}$ as an input, which has been created directly from raw data, and has no exploitable structure. and returns a tensor $\tensor{Y}$ which maximizes the interpretable and exploitable structure. The basic idea behind \method is to make a linear pass on the mode for which the granularity is suboptimal, and using a number of intuitive and locally optimal criteria for goodness of structure (henceforth referred to as {\em utility functions}), we greedily decide whether a particular slice across that mode needs to be aggregated\footnote{For the purposes of our work, we use matrix addition as aggregation of slice but this might not be the case and would depend on the problem domain. Other aggregation functions that can be used are OR, min, max, depending on the application domain (e.g., binary data).} into an existing slice or contains good-enough structure to stand on its own. 
\method can choose from a number of intuitive utility functions which are based on different definitions of good quality in matrices. 
\subsection{The \method algorithm}
Algorithm \ref{alg:algo1} gives a high level overview of \method. More specifically, the algorithm takes a three-mode tensor $\tensor{X}$ of dimension $I\times J \times K$ as an input and loops over all the $K$ slices of tensor $\tensor{X}$. Two slices next to each other get aggregated into a single slice if a certain utility function {\em has stabilized}, i.e., if aggregating the two slices does not offer any additional utility (larger than a particular threshold), then the second slice should not be aggregated with the first, and should mark the beginning of a new slice. 

Consider a three-mode tensor $\tensor{X}$ with time as third mode of dimension $I \times J \times K$ is ran through \method with a particular utility function. Our algorithm iterates over the time mode ($K$ slices) and aggregates slices as decided by the utility function. \method is agnostic to utility function used. Let us consider a slice that has been aggregated into a single slice from indices $i$ to $j-1$ called previous slice and another aggregated slice from indices $i$ to $j$ called a candidate slice. Both previous and candidate slice are passed to utility function separately to obtain a value each called previous and current value respectively. These values are compared (line $5$ in algorithm \ref{alg:algo1})  to decide whether $j^{th}$ slice is absorbed(line $6$ in algorithm \ref{alg:algo1}) into previous slice or previous slice has stabilized and entry is added in $W$ to indicate which indices of tensor $\tensor{X}$ are aggregated together(line $8-9$ in algorithm \ref{alg:algo1}). Now $j^{th}$ slice becomes the previous slice and aggregated slice of $j$ and $j+1$ become the candidate slice, the whole process is repeated until all the slices are exhausted.
\begin{center}
  \begin{algorithm}[h]
     \caption{\method} \label{alg:algo1}
	 \begin{algorithmic}[1]
      \REQUIRE Tensor $\tensor{X}$ of dimension $I \times J \times K$ 
      \ENSURE Tensor $\tensor{Y}$ of dimension $I \times J \times K_1 $and matrix $\mathbf{W}$ of size $K_1\times K$
      \STATE $i=1; j=2$
      \STATE $previousValue = UtilityFunction(X(:,:,i))$ \\
      \WHILE{$j\leq K$}
      		\STATE $currentValue = UtilityFunction(sum(X(:,:,i:j),3)$\\
      			\IF{$previousValue \lesseqqgtr currentValue$}
	      			\STATE {j = j+1} \COMMENT{Aggregate Slice} \\  
      			\ELSE 	

		      		\STATE	\COMMENT{Create a New Slice}\\ {Add a row in $\mathbf{W}$ with value as 1 for indices $i$ to $j-1$.} \\
      				\COMMENT{Update indices for next candidate slice}
      				\STATE	$i=j; j=j+1;$ 
      		        \STATE previousValue = UtilityFunction(X(:,:,i));
      		     \ENDIF
      		      
       \ENDWHILE
      			\STATE $\tensor{Y} = \tensor{X} \times_{3} \mathbf{W}$ \\
      		    \STATE return $\tensor{Y}$ and $\mathbf{W}$	
     \end{algorithmic}
  \end{algorithm}
\end{center}

Note that \method's complexity is {\em linear} in terms of the slices $K$ of the original tensor, and its overall complexity depends on the specific utility function used (which is called $O(K)$ times). 

\subsubsection{Utility functions:} 
\label{utility}
In this subsection, we summarize a number of intuitive utility functions that we are using in this paper. This list is by no means exhaustive, and can be augmented by different functions (or function combinations) that capture different elements of what is good structure and can be informed by domain-specific insights.

\begin{enumerate}
	\item \textbf{Norm:} We use multiple norm types to find adaptive granularity of a tensor. For a given threshold, if rate of change of norm between previous and candidate slice is less than the threshold, candidate slice is not selected. Our assumption in this case is no significant amount of information is being added to previous slice and is considered to have been stabilized. Matrix $\mathbf{W}$ is updated accordingly with indices of the previous slice (aggregated slices in previous slice). Otherwise the candidate slice is selected and the process continues until all the slices are exhausted. Different norms demonstrated in this work are Frobenius, 2-norm, and Infinity norm.
	\item \textbf{Matrix Rank:} In case of matrix rank, we focus on the $95\%$ reconstruction rank, which is typically much lower than the full rank of the data, but captures the essence of the number of components within the slice. In this case, we consider previous slice to be stabilized if the matrix-rank of previous slice decreases by addition of new slice, no more slices are added and an entry in matrix $\mathbf{W}$ is added. We keep aggregating slices if the matrix-rank of the slice is increasing or remains constant.
	\item \textbf{Missing Value Prediction:} If a piece of data has good structure, when we hide a small random subset of the data, the remaining data can successfully reconstruct the hidden values, under a particular model that we have chosen. To this end, we employ a variant of matrix factorization based collaborative filtering \cite{koren2009collaborative} as a utility function to see how good is the aggregated matrix in predicting certain percent of missing values. This utility function takes percent of missing value as a parameter, hides those percent of non zeros values in the matrix. Our implementation of matrix factorization with Stochastic Gradient Descent tries to minimize the loss function:
$
	\min_{\mathbf{U,V}}  \sum_{i,j \in \Omega } \mathcal{RMSE} \left(   \mathbf{A_{ij}} - \mathbf{U_{i, :}}\cdot\mathbf{V_{: ,j}} \right)
$
	where $\mathbf{A}$ is a given slice, $\mathbf{U,V}$ are factor matrices for a given rank (typically chosen using the same criterion as the matrix rank above), and $\Omega$ is the set of {\em observed} (i.e., non-missing) values. In order to create a balanced problem, since we are dealing with very sparse slices, we conduct {\em negative sampling} where we randomly sample as many zero entries as there are non-zeros in the slice, and this ends up being the $\Omega$ set of observed values.
\end{enumerate}

\subsection{The \ibPlus algorithm}
\method algorithm returns a tensor $\tensor{Y}$ as an output which is considered to have an exploitable and better structure than the input tensor $\tensor{X}$. The idea behind \ibPlus is to recursively feed the output back to \method until the third mode is reduced to a single slice (matrix) or the dimenison of third mode does not change. 
\method algorithm returns a tensor associated with each utility function. So if we used 5 utility functions, we would get 5 tensors associated with each of them. Now we select the tensor with highest CP Fit(see \ref{exp:measure}), use that as input for \method and we repeat this process until the stopping condition is met. The output of each iteration is a candidate tensor. At the end we have multiple tensors (one for each iteration) which has different temporal resolutions, which can help us get a tensor with optimal resolution based on the evaluation measures used. Algorithm \ref{alg:algo2} describes the process discussed in this section.

\begin{center}
  \begin{algorithm}[h]
  	 \small
     \caption{\ibPlus} \label{alg:algo2}
	 \begin{algorithmic}[1]
      \REQUIRE Tensor $\tensor{Y}$ of dimension $I \times J \times K$
      \ENSURE One Tensor for each iteration	
      \WHILE{$K\leq 1$}
            \FORALL{ Uitlity Functions}
      		\STATE  [$\tensor{Z}$, $\mathbf{W}$] = \method($\tensor{Y}$)\\
      		\ENDFOR
      		\STATE Select $\tensor{Z}$ with the best Realtive fit \\
      		\COMMENT{Third mode dimension}
      		\STATE $K_1 = size(\tensor{X},3)$
      		
      		\IF{$K_1 == K$}
	      		\STATE break;\\
      			\ELSE 
      			\STATE $K = K1$ \\
      			\STATE $\tensor{Y} = \tensor{Z} $\\
      		\ENDIF
      		
       \ENDWHILE
      		
      		    \STATE return one Tensor for each iteration	
     \end{algorithmic}
  \end{algorithm}
\end{center}

	\section{Experimental Evaluation}
\label{sec:experiments}
In this section we present a thorough evaluation of \ibPlus using variety of data, including synthetic, semi-synthetic and real data. We empirically evaluate our analysis using a number of criteria described in detail below. We implement our method in Matlab using tensor toolbox library \cite{TTB_Software}. 

\subsection{Evaluation measures} 
\label{exp:measure}

When formulating the problem, we did not specify a quality function $\mathcal{Q}$ to be maximized, nor did we use such a function in our proposed method. The reason for that is because we reserve the use of different quality functions as a form of evaluation. In particular, we use the two following notions of quality:
\begin{itemize}
    \item{\bf CP Fit:} To evaluate effectiveness of our method, we compute CP fit of the computed tensor for a particular rank with respect to the Input tensor.
    
\begin{equation}
Relative\ Fit= 1 - \Big(\frac{||\tensor{X}_{Input}-\tensor{X}_{computed}||_F}{||\tensor{X}_{Input}||_F}\Big)
\end{equation}

	\item {\bf CORCONDIA}: To evaluate the {\em interpretability} of the resulting tensor we employ Autoten \citep{papalexakis2016automatic} that given a tensor and some estimated tensor rank, returns a CORCONDIA \citep{bro2003new} score and low rank that provides best attainable tensor decomposition quality in a user-defined search space. 
\end{itemize}

We should note at this point that the two quality measures above are far from continuous and monotonic functions, thus we do not expect that our method progresses the quality will monotonically increase. Thus, we calculate the quality for the final solution of \ibPlus, and we reserve investigating whether monotonic and well-behaved quality functions exist for future work. 

In our experiments we used 5 utility functions (see \ref{utility}) namely Frobenius norm, 2-norm, Infinity norm, Matrix Rank and Missing Value Prediction. In case of synthetic datasets we ran all the utility functions once except for Missing Value Prediction which we ran for 10 times. In case of both semi-synthetic and real datasets, in the interest of computational efficiency, we ran all the utility functions once. 

\subsection{Baseline methods}
A naive way to find tensor $\tensor{Y}$ can be by aggregating time mode based on some fixed intervals. If time granularity was in milliseconds, then combining one thousand slices to form slices of seconds granularity reducing the third dimension of tensor $\tensor{X}$ from $K$ to $K/1000$. This can be applied incrementally from seconds to minutes and so on to find a tensor which has some exploitable structure.  We compare the resulting tensor $\tensor{Y}$ determined by \method against tensors constructed with fixed aggregations. For fixed aggregation we aggregate the temporal with window size of $10,\ 100$ and $1000$ for synthetic data. For semi-synthetic and real datasets we use appropriate time windows accordingly.
\subsection{Performance for synthetic data}
\noindent{\bf Creating synthetic data}:
In order to create synthetic dataset, we follow a two-step process,
\begin{enumerate}
    \item We create a random sparse tensor of specific sparsity.
    \item Subsequently, we randomly distribute (drawn from uniform distribution) non zero entries in each slice over some fixed number of slice as explained in below example.
\end{enumerate}
\textbf{Example:} Consider a three-mode tensor $\tensor{X}$ of dimension $I\times J \times K$, for purpose of this example consider $K=4$ as shown in fig \ref{fig:syntheticDataset}. Now for each slice of size $I\times J$, distribute randomly (drawn from Uniform distribution) all the non-zeros entries across $W$ slices preserving the $I$ and $J$ indices, creating a tensor of size $I\times J \times W$. Now append all the tensors in the same order as they appeared in the original tensor, we get a resulting tensor of size $I \times J \times 4W$, which is used as an input for \method. So if the original tensor if of size $I\times J \times K$ and bucket size $W$, the resulting tensor is of size $I\times J \times KW$ approximately\footnote{The number of slice can be less than $KW$, since slice for each non-zero value is selected randomly, there can be a case where a slice is not selected}.
Table \ref{table:tsyndataset} shows the synthetic data used for experiments.
\begin{figure}[h]
	\begin{center}
		\includegraphics[width = 0.80\textwidth]{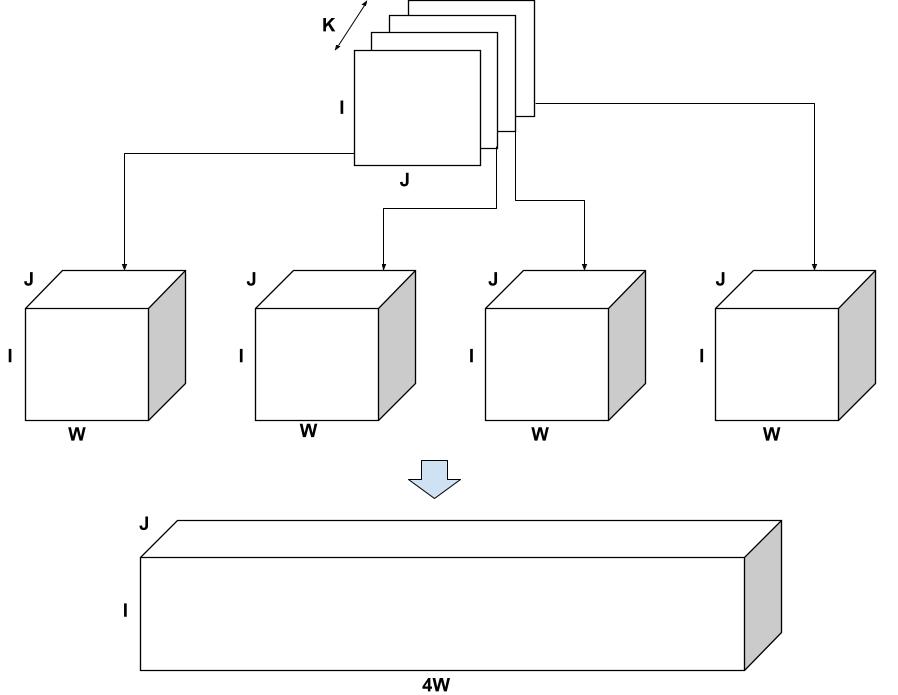}
        \caption{\small{Creating Synthetic Data}}
		\label{fig:syntheticDataset}
	\end{center}
\end{figure}

\begin{table*}[h!]
\small
	\begin{center}
		\begin{tabular}{ |c|c|c|c|c|}
			\hline
			Dataset & Original Dimension & Window size($W$) &  Approximate Final Dimension & Number of datasets  \\
			\hline
			SD1 & {$100 \times 100 \times 10$} & 50 & {$100 \times 100 \times 500$} & 10\\
		      \hline
			SD2 & {$100 \times 100 \times 100$} & 50 & {$100 \times 100 \times 5000$} & 10\\
		      \hline
		\end{tabular}
		\caption{Table of Synthetic Datasets analyzed}
		\label{table:tsyndataset}
	\end{center}
\end{table*}
\begin{table*}[h!]
	\begin{center}
		\begin{tabular}{ |c|c|c|c|}
			\hline
			Dataset & Original Dimension & Window size($W$) &  Approximate Final Dimension  \\
	
		      \hline
		    Enron Weekly & {$184 \times 184 \times 44$} & 4 & {$184 \times 184 \times 176$}\\
		    \hline
		      Enron Daily & {$184 \times 184 \times 44$} & 30 & {$184 \times 184 \times 1320$}\\
		      \hline
		       Enron Hourly & {$184 \times 184 \times 44$} & $720$ & {$184 \times 184 \times 31680$}\\
		       \hline
		\end{tabular}
		\caption{Table of Semi-synthetic Datasets analyzed}
		\label{table:tsemisyndataset}
	\end{center}
\end{table*}

\noindent{\bf Results for synthetic data}: In order to evaluate the performance of \ibPlus, we measure CORCONDIA and fit on 10 synthetic datasets for both type of datasets as mentioned in table \ref{table:tsyndataset}. In interest of conserving space we only show one set of results for both synthetic datasets. The leftmost part of the Figure \ref{fig:synthetic1} and \ref{fig:synthetic2} shows the best fit at end of each iteration. The number on top of the dots represent the dimension of the third mode after each iteration. The dotted line in the plot show the fit of the input tensor and fixed intervals tensor.\footnote{The number in the parenthesis represents the dimension of the third mode for that tensor.}
The rightmost part of the Figure \ref{fig:synthetic1} and \ref{fig:synthetic2} shows the CORCONDIA computed at the end of each iteration and absolute change of CORCONDIA. Absolute change of CORCONDIA is computed as shown below:
$$abs(corcondia(j+1) - corcondia(j))$$

The dotted line in the plot represent CORCONDIA value for the fixed intervals tensor.
When there is sudden drop in the value of CORCONDIA we consider the iteration before as an suitable candidate for tensor analysis.  In the case of SD1 that would be iteration number 2 and resulting tensor of size $100 \times 100 \times 8$. In the case of SD2 that would also be iteration number 2 and resulting tensor of size $100 \times 100 \times 57$

\begin{figure}[ht]
	\begin{center}
	    \includegraphics[width =0.8\textwidth]{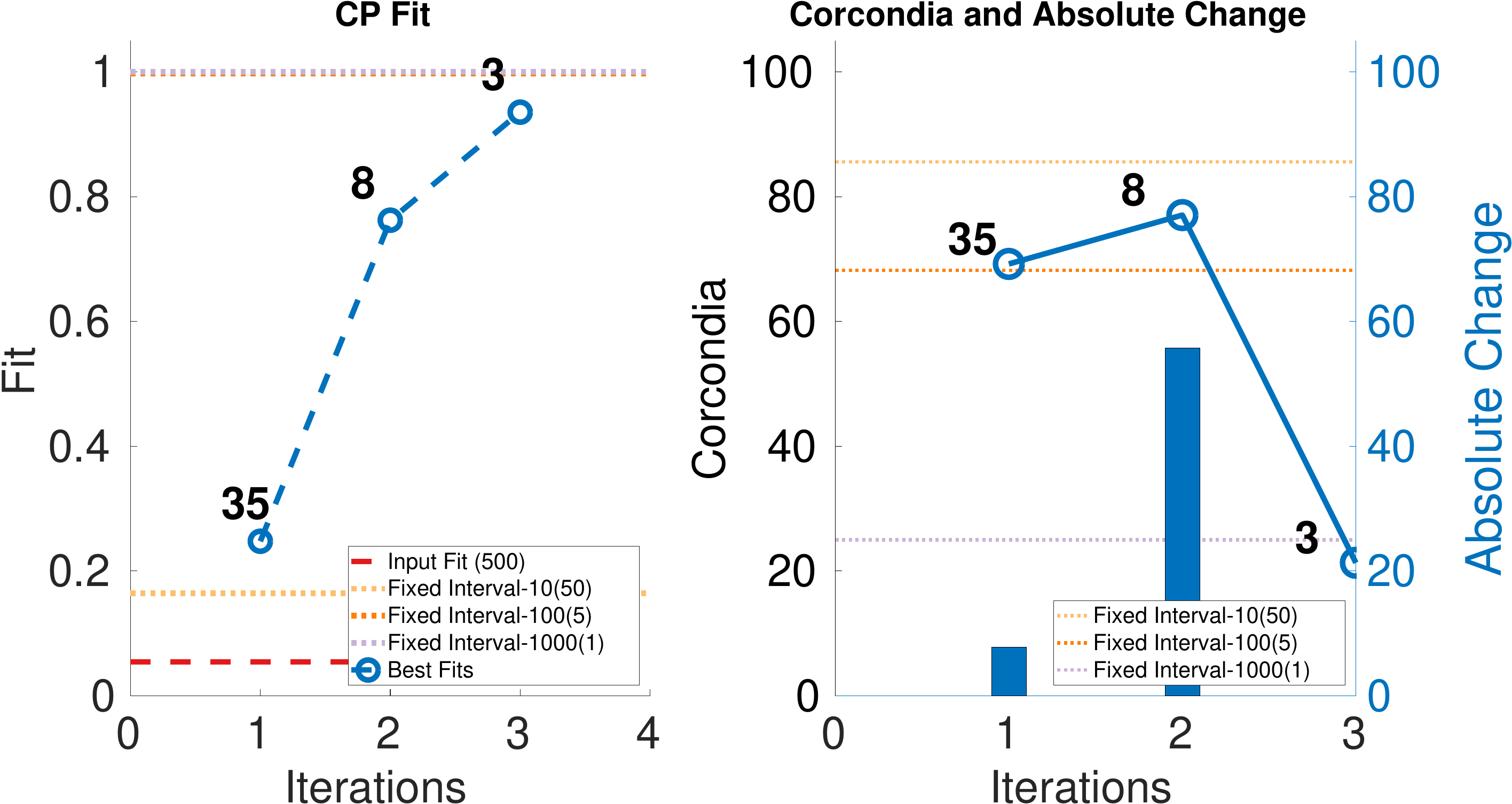}
        \caption{CP fit and Corcondia of best fit tensor \& its absolute change at each iteration for SD1.}
		\label{fig:synthetic1}
	\end{center}
\end{figure}
\begin{figure}[ht]
	\begin{center}
		\includegraphics[width =0.8\textwidth]{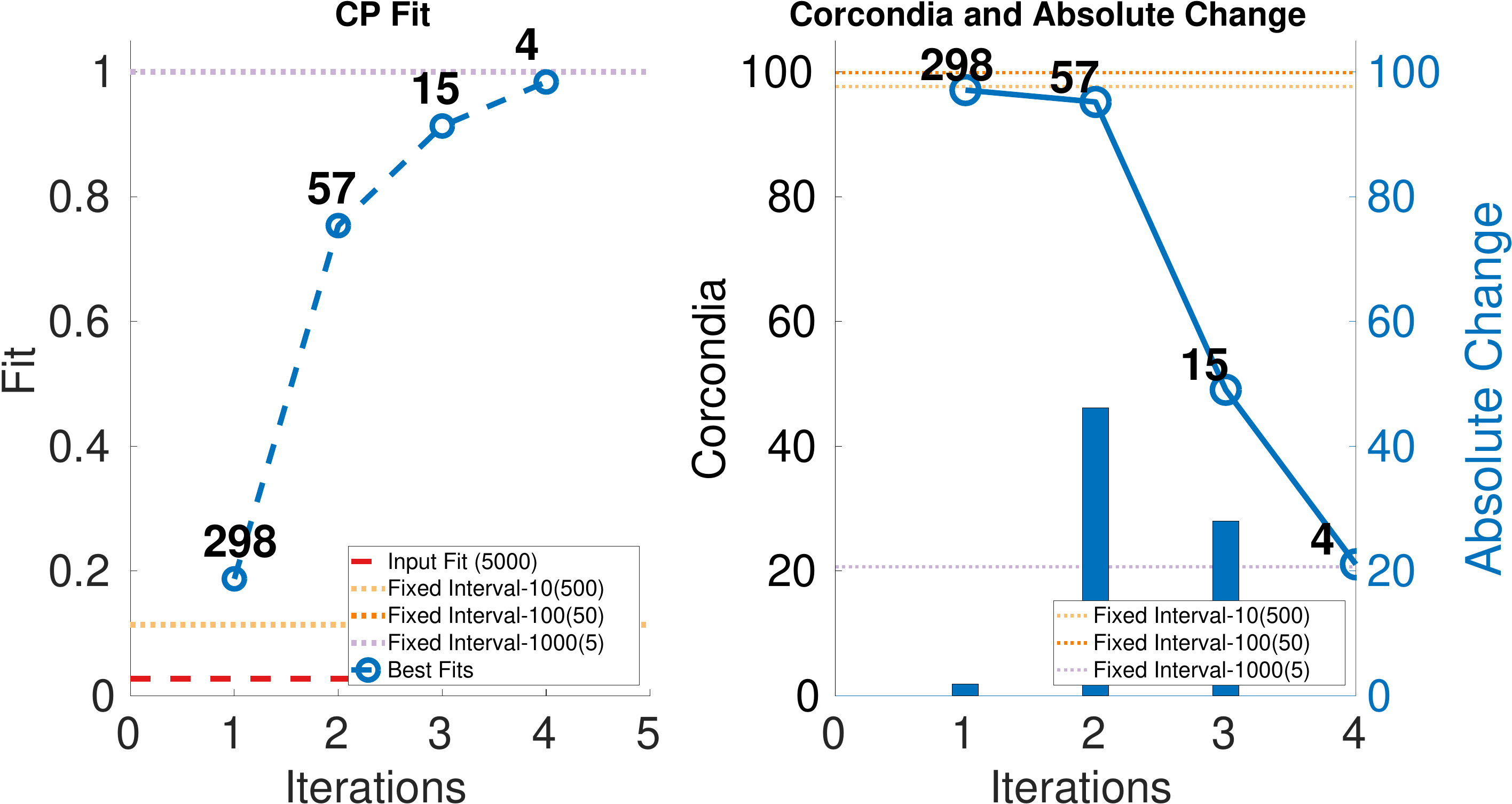}
		\caption{CP fit and CORCONDIA of best fit tensor \& its absolute change at each iteration for SD2.}
		\label{fig:synthetic2}
	\end{center}
\end{figure}

\subsection{Performance for semi-synthetic data}
\noindent{\bf Creating semi-synthetic data}: Here we used Enron dataset\cite{enrondataset, bader2007temporal}, which is dataset of number of email exchanges between employees spread over 44 months. Each month is represented by a matrix. To create the semi-synthetic data, we use the step 2 as described in the generation of synthetic case. 
We take the non-zero elements and randomly distribute non zero entries in each slice over some fixed number of slice. For this dataset we converted the monthly data into weekly, daily and hourly data. Non-zero entries in each slice was distributed over 4 different candidate slices of monthly (roughly aprroximating 4 weeks as a month). In the case of daily each slice of monthly data was distributed over 30 different slices as mentioned in table \ref{table:tsemisyndataset} and finally in the case of hourly each non zero entry in the monthly slice was distributed over 720 slices ($24 \times 30$).

\noindent{\bf Results for semi-synthetic data}: The leftmost parts of Figures \ref{fig:enronWeekly}, \ref{fig:enronDaily} and \ref{fig:enronHourly} show the fit of different iterations and rightmost part of the figures Figures \ref{fig:enronWeekly}, \ref{fig:enronDaily} and \ref{fig:enronHourly} show the CORCONDIA computed at different iterations. In the case of Enron Weekly we see a sudden drop in CORCONDIA after iteration 1 as shown in Figure \ref{fig:enronWeekly} and corresponding tensor is of size $184 \times 184 \times 17$.
In the case of Enron Daily we don't see a significant change in CORCONDIA values in two iterations and corresponding tensors are of size $184 \times 184 \times 78$ and $184 \times 184 \times 5$ giving us tensors of different granularity.

In the case of Enron Hourly we see a drop in CORCONDIA after iteration 1 and 2 as shown in Figure \ref{fig:enronHourly}. In this case practitioner can make choice between a tensor of resolution $184\times 184\times 469$ or $184\times184\times34$ depending on what evaluation metric they value more, fit, CORCONDIA or both. Tensor after iteration 2 ($184\times184\times34$) seems to have good score for both fit and CORCONDIA whereas Tensor after iteration 1 has good COROCONIA score but not a good CP fit.

\begin{figure}[ht]
	\begin{center}
		\includegraphics[ width=0.8\textwidth]{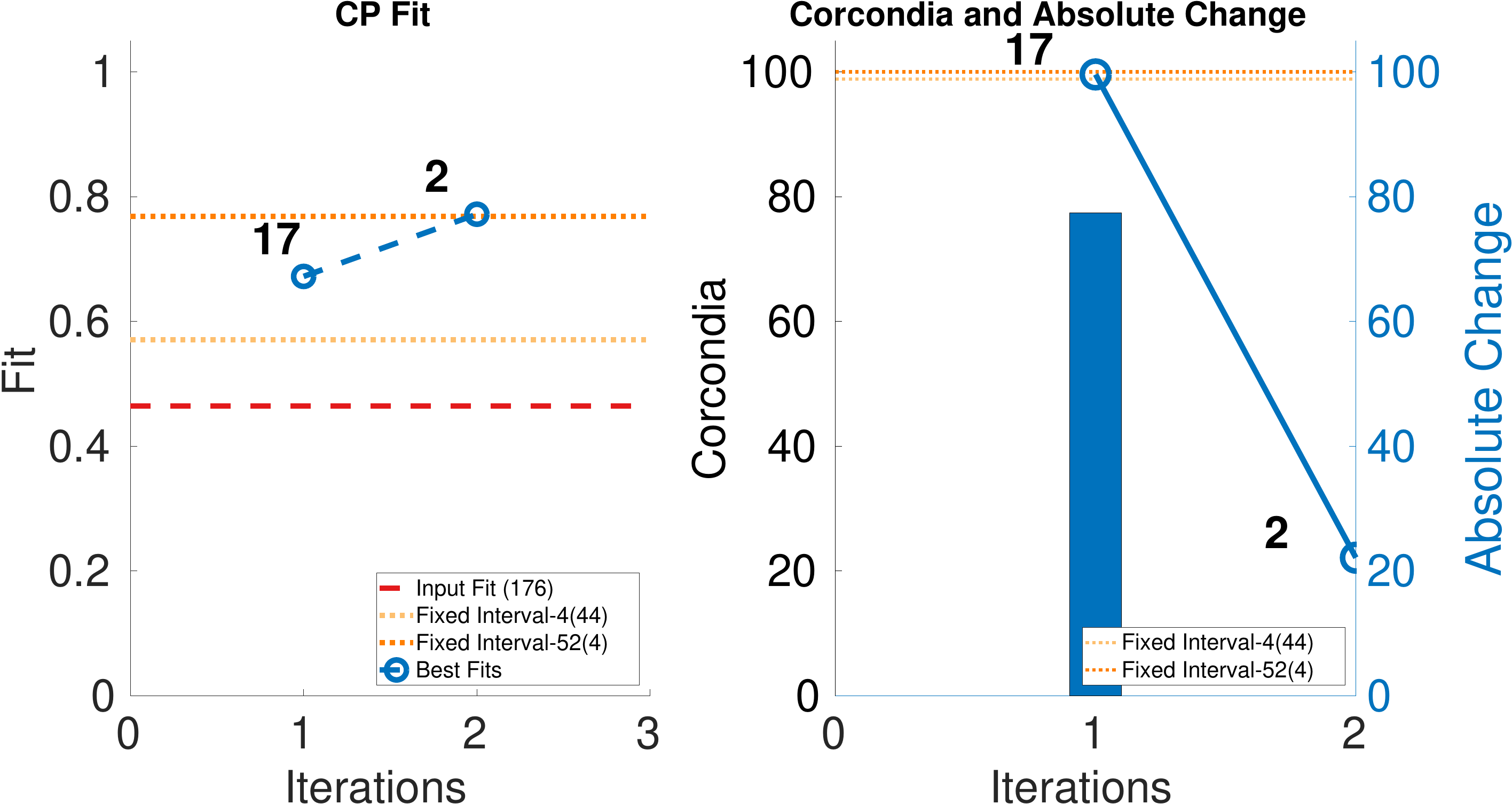}
		\caption{CP fit and CORCONDIA of best fit tensor \& its absolute change at each iteration for Enron Weekly.}
		\label{fig:enronWeekly}
	\end{center}
\end{figure}
\begin{figure}[ht]
	\begin{center}
		\includegraphics[width = 0.8\textwidth]{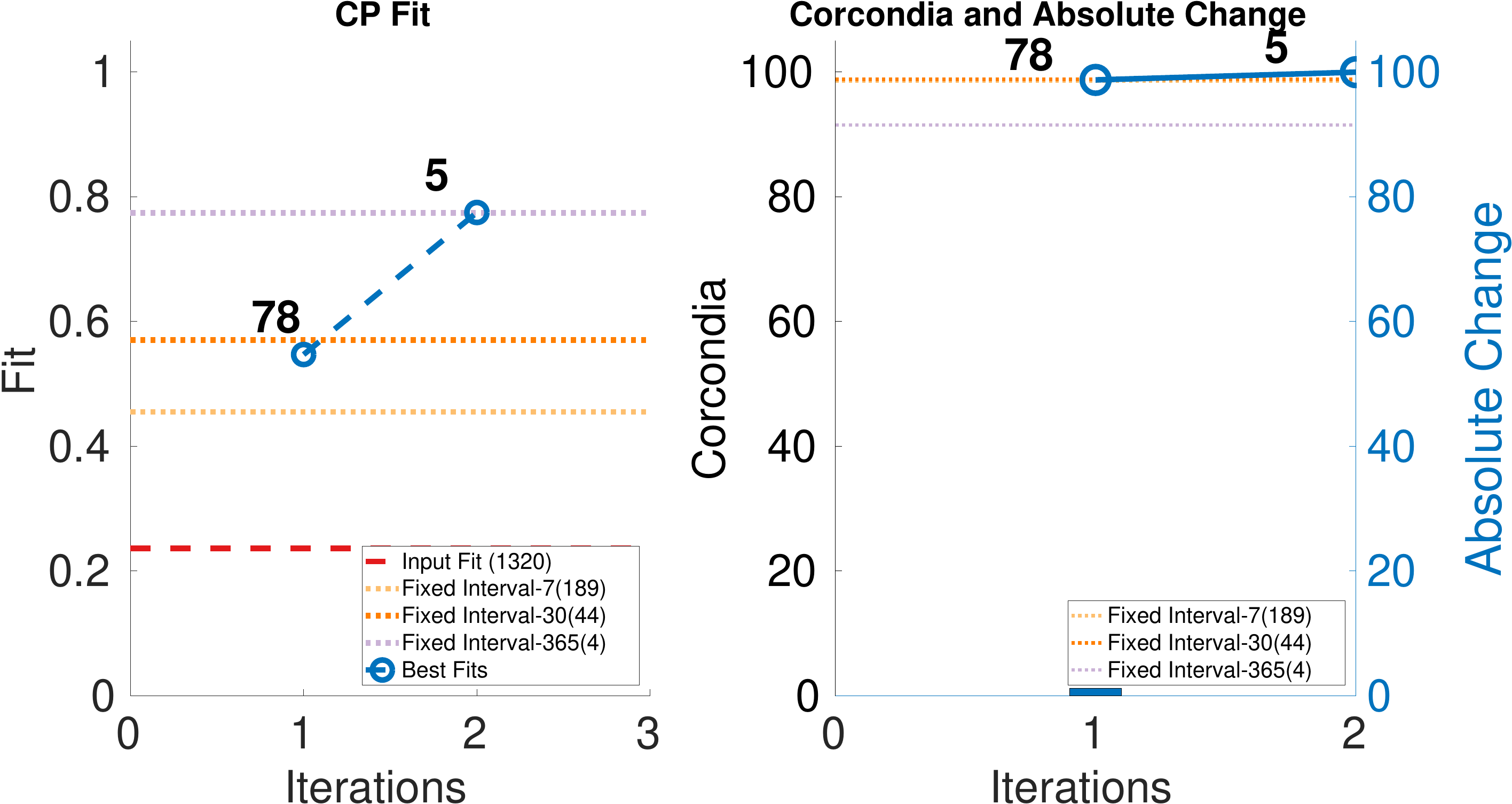}
		\caption{CP fit and CORCONDIA of best fit tensor \& its absolute change at each iteration for Enron Daily.}
		\label{fig:enronDaily}
	\end{center}
\end{figure}

\begin{figure}[ht]
	\begin{center}
		\includegraphics[width = 0.8\textwidth]{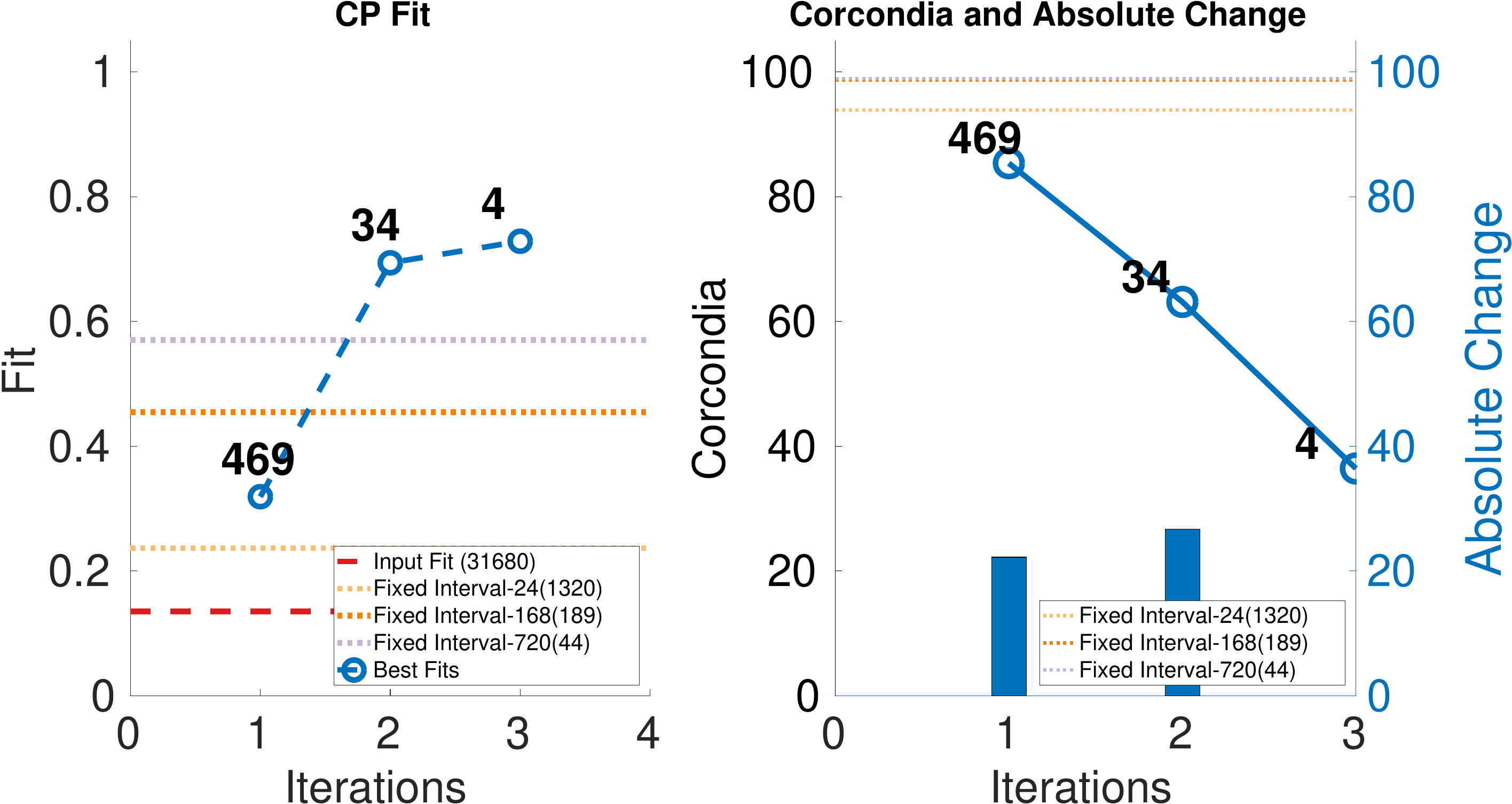}
		\caption{CP fit and CORCONDIA of best fit tensor \& its absolute change at each iteration for Enron Hourly.}
		\label{fig:enronHourly}
	\end{center}
\end{figure}

\subsection{Data mining case study}
\hfill\\
\noindent{\bf Chicago crime dataset}: For our case study we use a dataset provided by the city of Chicago\footnote{\url{https://data.cityofchicago.org/Public-Safety/Crimes-2001-to-Present/ijzp-q8t2}} that records different types of crime committed in different areas of the city over a period of time\cite{frosttdataset2017}. The tensor we create has modes (area, crime, timestamp), where ``community area'' and ``crime'' are discretized by the city of Chicago and ``timestamp'' is the coarsely aggregated (hourly) timestamp. The dates that we focused on span a period of $7$ years, between December 13, 2010 to December 11, 2017.

\begin{figure}[htp]
	\begin{center}
		\includegraphics[width =0.8\textwidth]{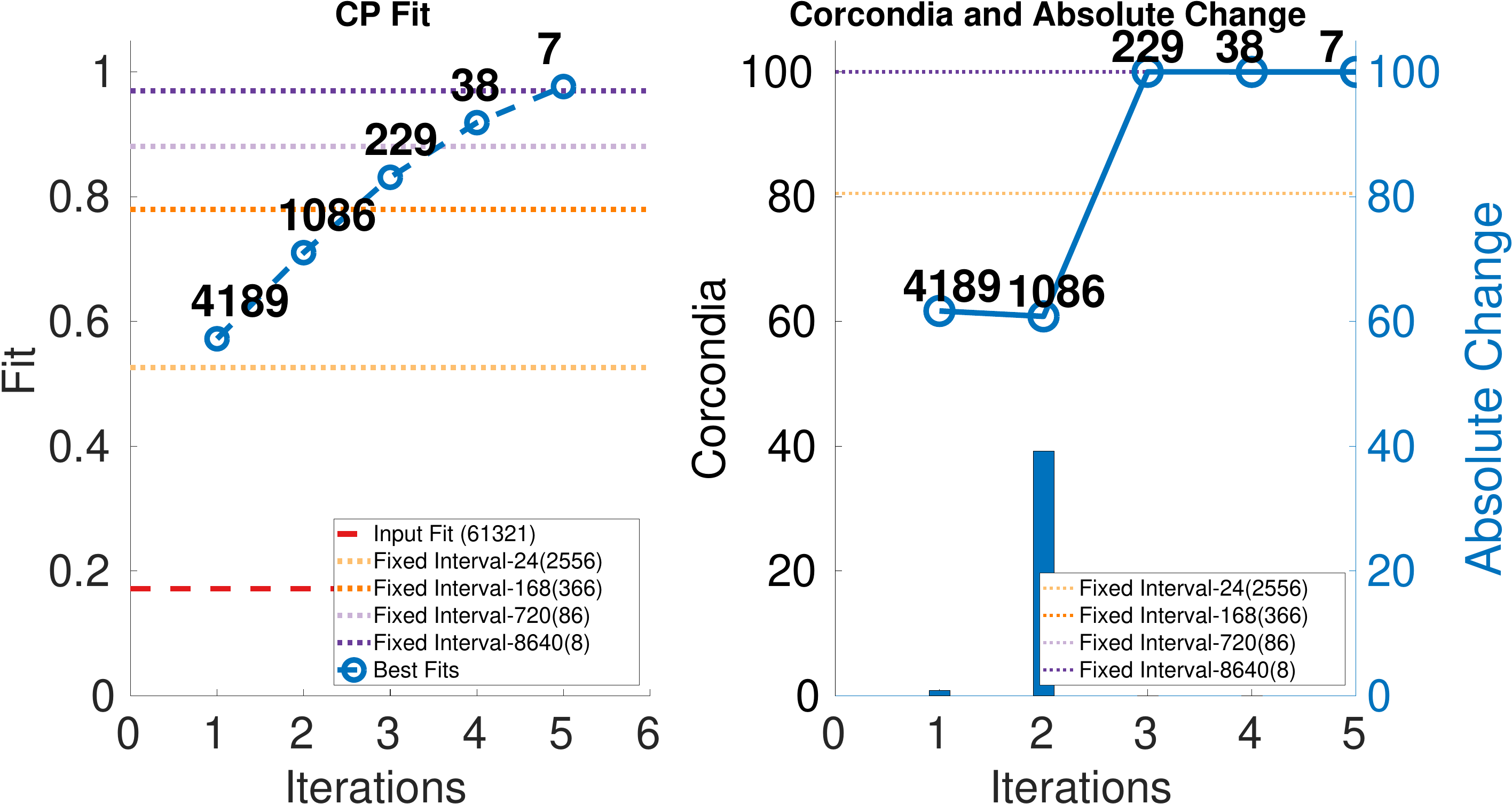}
		\caption{CP fit and CORCONDIA of best fit tensor \& its absolute change at each iteration for Chicago Crime Dataset.}
		\label{fig:chicago}
	\end{center}
\end{figure}
We ran \ibPlus on this dataset which of size $77\times 32 \times 61321$, and in right most part of the Figure \ref{fig:chicago} we show its CORCONDIA for each iterations and we observe that iterations 3,4 5 has high value of CORCONDIA, which would suggest they offer a resolution with an exploitable structure. Iteration 1 and 2 also have decent CORCONDIA value.
Given these two range of CORCONDIA values, we decided to drill down and look into the actual tensor components that can be extracted from those different tensors. In the interest of space, we took the tensor returned by iteration 2 as $\tensor{X}_1$, the tensor $\tensor{X}_2$ and tensor $\tensor{X}_3$ returned by iteration 3 and 4 respectively. Tensor $\tensor{X}_1$ contains three high-quality components, whereas $\tensor{X}_2$ and  $\tensor{X}_3$ contains two.

Figure \ref{fig:chicago_patterns} shows the two different sets of patterns\footnote{We omit plotting the temporal mode since we lack external information that we can potentially correlate it with, however, an analyst with such side information can find the different time resolutions of $\tensor{X}_1$ and $\tensor{X}_2$ useful.}: interestingly, factor 1 of $\tensor{X}_1$ and  factor 1 of $\tensor{X}_2$ pertain to the similar spatial and criminal pattern. In interest of space we omit the plots for $\tensor{X}_3$ but we observed that both factors of tensor $\tensor{X}_2$ and $\tensor{X}_3$ pertain to the similar spatial and criminal patterns. In summary, tensors $\tensor{X}_1$, $\tensor{X}_2$ and $\tensor{X}_3$ capture similar interpretable patterns over different temporal resolutions.

\begin{figure*}[htp]
	\begin{center}	
		\subfigure[Analysis of Iteration-2]{\includegraphics[width = 0.8\textwidth]{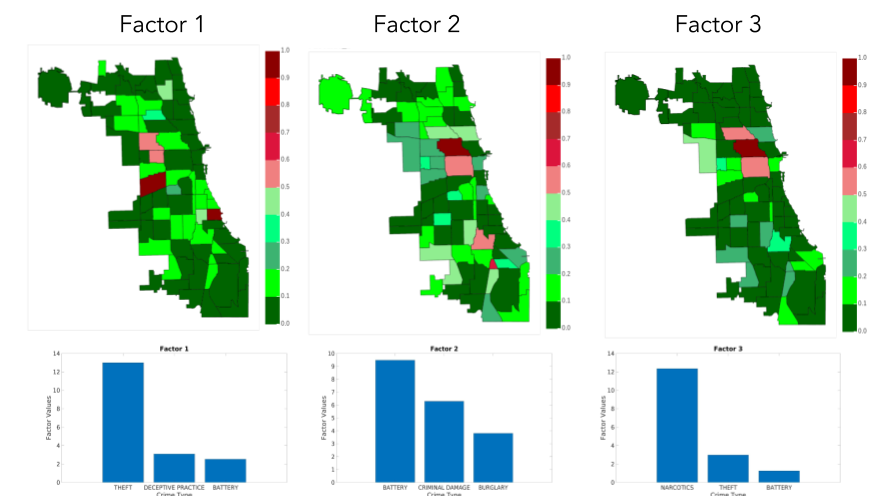}}
		\subfigure[Analysis of Iteration-3]{\includegraphics[width = 0.8\textwidth]{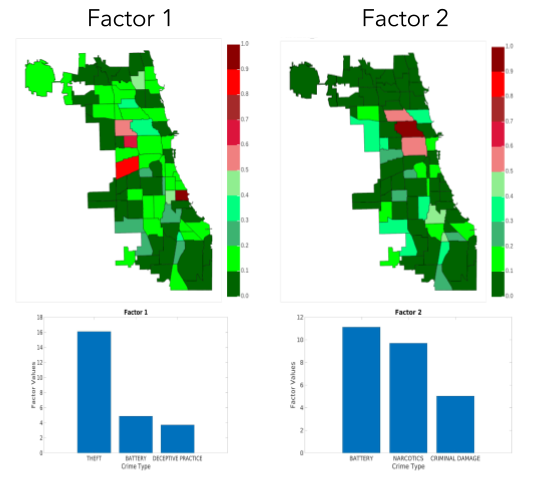}}
		\caption{Analyzing the Chicago data from two different resolutions discovered}
		\label{fig:chicago_patterns}
	\end{center}
\end{figure*}

\noindent{\bf Comparison against fixed aggregation}: A natural question is whether the results are qualitatively ``better'' than the ones by a fixed aggregation. Answering this question heavily depends on the application at hand, however, here we attempt to quantify this in the following way: intuitively, a good set of components offers more {\em diversity} in how much of the data it covers. For instance, a practitioner would prefer a set of results for the Chicago crime dataset where the components span most of the regions of the city and uncover diverse patterns of crime, over a set of components that seem to uncover a particular type of crime. Even though there may be a number of confounding factors, aggregating on a regular time interval may be very good in capturing periodic activity (in this example, crime that exhibits normal periodicity that happens to coincide with the aggregation resolution we have chosen), whereas aggregating adaptively may help discover structure that is more erratic and more surprising. In order to capture this and test this hypothesis, we compute the coverage of entities for the first and second mode of the tensor (i.e., areas of Chicago and crime types in this example) in all the discovered components: for each component, we measure the top-k entities, and through that we compute the empirical probability distribution of all entities in the results. A more preferable set of results will have a higher coverage, resulting in a distribution with higher entropy. In Table \ref{tbl:chicagoEntropy} we show the entropy for both modes 1 and 2 for \ibPlus and for the different fixed aggregations (averaged over 10 different runs), where \ibPlus overall offers more diverse patterns in both space and criminal activity.
 
\begin{center}
	\begin{table*}[h]
		\ssmall
		\begin{tabular}{|c|c|c|c|c|c|c|c|c|c|}
			\hline
			\multirow{1}{*}  & Iteration & Iteration & Iteration & Iteration & Iteration & Fixed & Fixed & Fixed & Fixed \\
			&1&2&3& 4 & 5 & Interval-24 & Interval-168 & Interval-720 & Interval-8640\\ \hline
			Area	& \textbf{2.8554} & 2.6810 & 2.5850 & 2.5850 & 2.5850 & \textcolor{red}{2.7255} & 2.5850 & 2.5850 & 2.5850  \\ \hline
			Crime	& \textbf{2.8783} & 2.7255 & 2.2516 & 2.2516 & 2.0850 & \textcolor{red}{2.4362} & 2.2516 & 2.2516 & 2.1183  \\ \hline
		\end{tabular}
		\caption{\small{Entropy of top-3 components in factors for area and crime type}}
		\label{tbl:chicagoEntropy}
	\end{table*}
\end{center}
	\section{Related Work}
\label{sec:related}

To the best of our knowledge, this is the first attempt at formalizing and solving this problem, especially as it pertains in the tensor and multi-aspect data mining domain.
Nevertheless, there has been significant amount of work on temporal aggregations in graphs
\cite{soundarajan2016generating,sulo2010meaningful,sun2007graphscope} and in finding communities in temporal graphs \cite{gorovits2018larc}. In the graph literature, the closest work to ours is \cite{soundarajan2016generating}, in which the authors looks at aggregating stream of temporal edges to produce sequence of structurally mature graphs based on a variety of network properties.

In the tensor literature, Almutairi et al. \cite{almutairi2021prema} are solving the inverse of this problem, where the goal is to disaggregate a tensor. Concurrently to our work, Kwon et al. \cite{DBLP:conf/icde/Kwon0LS21} develop a streaming CP decomposition that works on the original granularity of the data, instead of preprocessing the tensor in order to identify one or more optimal aggregations. We reserve a full investigation of our problem formulation and Kwon et al. \cite{DBLP:conf/icde/Kwon0LS21} for future work.

	\section{Conclusions}
\label{sec:conclusions}
In this paper we are, to the best of our knowledge, the first to define and formalize the \ourproblem problem in constructing a tensor from raw sparse data. We demonstrate that an optimal solution is intractable and subsequently proposed \method and \ibPlus, a practical solution that is able to identify good tensor structure from raw data, and construct tensors from the same dataset that pertain to multiple resolutions. Our experiments demonstrate the merit of \ibPlus in discovering useful and high-quality structure, as well as providing tools to data analysts in automatically extracting multi-resolution patterns from raw multi-aspect data. In future work we will work towards extending \method in cases where more than one modes is \ourproblem (naively one can apply \method to each mode sequentially, but this disregards joint variation across modes), and extend \method for higher-order tensors.

\section{Acknowledgements}
{\scriptsize
	Research was supported by the National Science Foundation under CAREER grant no. IIS 2046086. Research was also supported by the Department of the Navy, Naval
Engineering Education Consortium under award no. N00174-17-1-
0005. This research was sponsored by the Combat Capabilities Development Command Army Research Laboratory and was accomplished under Cooperative Agreement Number W911NF-13-2-0045 (ARL Cyber Security CRA). The views and conclusions contained in this document are those of the authors and should not be interpreted as representing the official policies, either expressed or implied, of the Combat Capabilities Development Command Army Research Laboratory or the U.S. Government. The U.S. Government is authorized to reproduce and distribute reprints for Government purposes not withstanding any copyright notation here on.
 }

\bibliographystyle{ACM-Reference-Format}
\bibliography{BIB/reference}

\end{document}